\documentclass[12pt,verbose=true,letterpaper,margin=1.5in]{article}
\usepackage{arxiv}

\usepackage[utf8]{inputenc}
\usepackage[vlined,linesnumbered,ruled,titlenotnumbered]{algorithm2e}
\usepackage{amsthm,amsmath,amsfonts,mathtools,amssymb}
\usepackage{diagbox}
\usepackage{xspace}
\usepackage{multirow, rotating}
\usepackage{adjustbox}

\newcommand{\gsemo}{GSEMO\xspace}
\newcommand{\semo}{SEMO\xspace}
\newcommand{\gsemod}{GSEMO$_D$\xspace}
\newcommand{\semod}{SEMO$_D$\xspace}
\newcommand{\oea}{\mbox{$(1 + 1)$~EA}\xspace}

\newcommand{\oneminmax}{\textsc{OneMinMax}\xspace}
\newcommand{\onemax}{\textsc{OneMax}\xspace}
\newcommand{\binval}{\textsc{BinVal}\xspace}
\newcommand{\jump}{\textsc{Jump}\xspace}
\newcommand{\lotz}{\textsc{LOTZ}\xspace}
\newcommand{\leadingones}{\textsc{LeadingOnes}\xspace}
\newcommand{\lo}{\textsc{LO}\xspace}
\newcommand{\trailingzeros}{\textsc{TrailingZeros}\xspace}
\newcommand{\tz}{\textsc{TZ}\xspace}

\newcommand{\R}{\mathbb{R}}
\newcommand{\N}{\mathbb{N}}
\newcommand{\bopt}{b_{\text{opt}}}

\DeclareMathOperator{\Bin}{Bin}

\newtheorem{theorem}{Theorem}
\newtheorem{lemma}{Lemma}

\usepackage{tikz}
\usetikzlibrary{patterns,patterns.meta}
\usepackage{pgfplots}
\pgfplotsset{compat=newest}
\pgfplotscreateplotcyclelist{tikzcycle}{%
    thick,blue,mark=triangle*\\ % k = 2
    thick,red,mark=*\\ % k = 4
    thick,green!70!black,mark=diamond*\\ % k = sqrt(n)
    thick,yellow!70!black,mark=square*\\ % k = n / 2
    thick,purple,mark=x\\ % k = n
    thick,blue,dashed,mark=triangle*\\ % k = 2
    thick,red,dashed,mark=*\\ % k = 4
    thick,green!70!black,dashed,mark=diamond*\\ % k = sqrt(n)
    thick,yellow!70!black,dashed,mark=square*\\ % k = n / 2
    thick,purple,dashed,mark=x\\ % k = n
    thick,blue,dashdotted,mark=triangle*\\ % k = 2
    thick,red,dashdotted,mark=*\\ % k = 4
    thick,green!70!black,dashdotted,mark=diamond*\\ % k = sqrt(n)
    thick,yellow!70!black,dashdotted,mark=square*\\ % k = n / 2
    thick,purple,dashdotted,mark=x\\ % k = n
}
\pgfplotscreateplotcyclelist{tikzcycle-notwo}{%
    thick,red,mark=*\\ % k = 4
    thick,green!70!black,mark=diamond*\\ % k = sqrt(n)
    thick,yellow!70!black,mark=square*\\ % k = n / 2
    thick,purple,mark=x\\ % k = n
    thick,red,dashed,mark=*\\ % k = 4
    thick,green!70!black,dashed,mark=diamond*\\ % k = sqrt(n)
    thick,yellow!70!black,dashed,mark=square*\\ % k = n / 2
    thick,purple,dashed,mark=x\\ % k = n
    thick,red,dashdotted,mark=*\\ % k = 4
    thick,green!70!black,dashdotted,mark=diamond*\\ % k = sqrt(n)
    thick,yellow!70!black,dashdotted,mark=square*\\ % k = n / 2
    thick,purple,dashdotted,mark=x\\ % k = n
}

\begin{document}

\author{Denis Antipov
\\LIP6, CNRS, Sorbonne Universit\'e
\\Paris, France
\And
Aneta Neumann
\\Optimisation and Logistics\\
School of Computer and Mathematical Sciences
\\The University of Adelaide
\\Adelaide, Australia
\And
Frank Neumann
\\Optimisation and Logistics\\
School of Computer and Mathematical Sciences
\\The University of Adelaide
\\Adelaide, Australia
\And
Andrew M. Sutton
\\University of Minnesota Duluth
\\Duluth, USA
}

\title{Runtime Analysis of Evolutionary Diversity Optimization on the Multi-objective (LeadingOnes,~TrailingZeros) Problem}

\maketitle

\begin{abstract}

    Diversity optimization is the class of optimization problems in which we aim to find a diverse set of good solutions. One of the frequently-used approaches to solve such problems is to use evolutionary algorithms that evolve a desired diverse population. This approach is called evolutionary diversity optimization (EDO).

    In this paper, we analyze EDO on a three-objective function LOTZ$_k$,
    which is a modification of the two-objective benchmark function
    (LeadingOnes, TrailingZeros). We prove that the GSEMO computes a
    set of all Pareto-optimal solutions in $O(kn^3)$ expected
    iterations. We also analyze the runtime of the GSEMO$_D$ algorithm
    (a modification of the GSEMO for diversity optimization) until it finds a population with the best possible diversity for two different diversity measures: the total imbalance and the sorted imbalances vector. For the first measure we show that the GSEMO$_D$ optimizes it in $O(kn^2\log(n))$ expected iterations (which is asymptotically faster than the upper bound on the runtime until it finds a Pareto-optimal population), and for the second measure we show an upper bound of $O(k^2n^3\log(n))$ expected iterations.

    We complement our theoretical analysis with an empirical study, which shows a very similar behavior for both diversity measures. The results of experiments suggest that our bounds for the total imbalance measure are tight, while the bounds for the imbalances vector are too pessimistic.
\end{abstract}


\section{Introduction}
\label{sec:intro}
Computing a diverse set of high quality solutions has recently become an important topic in the area of artificial intelligence and in particular in the field of evolutionary computation~\cite{DBLP:conf/aaai/HanakaK0KKO23,DBLP:conf/aaai/HanakaKKLO22,DBLP:conf/aaai/IngmarBST20}. Different approaches have been designed for using classical solvers in order to compute diverse sets of high quality solutions for problems in the areas of planning~\cite{DBLP:conf/aaai/0001S20} and satisfiability~\cite{DBLP:conf/gecco/NikfarjamR0023}.
Such problems are often met in practice, especially when there are some factors which are hard to formalize, such as politics, ethics or aesthetics. In these cases, the algorithm user would prefer to have several different good solutions rather than one single best solution in order to have an opportunity to choose among them. In practice this problem arises in, e.g., optimization of floor plans of buildings~\cite{DBLP:journals/bit/Galle89}, the cutting stock problem~\cite{Haessler1991CuttingSP}, space mission design~\cite{DBLP:conf/ppsn/SchutzeVC08}, quantum control~\cite{ShirEtAl2008} structural optimization in bridges~\cite{DBLP:conf/gecco/UlrichT11} and drug discovery~\cite{DBLP:journals/isci/YevseyevaLVIDE19}.

In contrast to standard single-objective optimization, where the search is performed in a space of potential solutions, diversity optimization works in a space of solution sets and is usually also harder from a computational complexity perspective~\cite{DBLP:conf/aaai/HanakaK0KKO23}. This makes it natural to use evolutionary algorithms (EAs) for solving such problems, since they are designed to evolve populations of solutions. This approach is called \emph{evolutionary diversity optimization} (EDO). 
Evolutionary diversity optimization aims at finding a set of solutions such that (i) all solutions meet a given quality threshold and (ii) the set of solutions has maximum diversity according to a chosen diversity measure.

In multi-objective problems, where the aim is to find a set of Pareto-optimal solutions which are diverse in their fitness, there also might be a need for a diversity of their genotype. For example, in~\cite{DBLP:journals/tog/MakaturaGS0M21} the authors designed optimal mechanical parts for different contexts with the aim of finding a good design for each context. This work inspired multi-objective quality diversity (QD) algorithms which are closely related to EDO~\cite{DBLP:conf/gecco/PierrotRBC22}. EDO would also be useful in this setting, since it allows to find a diverse Pareto-optimal set, where for each balance between the main objectives the decision maker could choose a good design for their context.

\subsection{Related work}
Feature-based EDO approaches that seek to compute a diverse set of solutions with respect to a given set of features have been carried out for evolving different sets of instances of the traveling salesperson problem~\cite{DBLP:journals/ec/GaoNN21} as well as evolving diverse sets of images~\cite{DBLP:conf/gecco/AlexanderKN17}. For these computations a variety of different diversity measures with respect to the given features such as the star discrepancy measure~\cite{DBLP:conf/gecco/NeumannGDN018} and the use of popular indicators from the area of evolutionary multi-objective optimization~\cite{DBLP:conf/gecco/NeumannG0019} have been studied.

Classical combinatorial optimization problems for which EDO algorithms have been designed to compute diverse sets of solutions include the traveling salesperson problem~\cite{DBLP:conf/gecco/NikfarjamBN021,DBLP:conf/foga/NikfarjamB0N21,DBLP:journals/telo/DoGNN22}, the traveling thief problem~\cite{DBLP:conf/gecco/NikfarjamN022a}, the computation of minimum spanning trees~\cite{DBLP:conf/gecco/Bossek021} and related communication problems in the area of defense~\cite{DBLP:conf/gecco/NeumannGYSCG023}.

Establishing theoretical foundations of evolutionary diversity optimization in the context of runtime analysis is a challenging task as it involves the understanding of population dynamics with respect to the given problem and the used diversity measure.
Initial studies have been carried out for classical benchmark problems in the area of evolutionary computation such as \onemax and \leadingones~\cite{DBLP:conf/gecco/GaoN14}.
For permutation problems such as the traveling salesperson problem and the quadratic assignment problem, runtime bounds have been provided for computing a maximal diverse set of permutations when there is no restriction on the quality of solutions~\cite{DBLP:journals/telo/DoGNN22}. In the context of optimization of monotone submodular problems with given constraints, results on the approximation quality of diversifying greedy approaches that produce a diverse population have been provided in~\cite{DBLP:conf/gecco/NeumannB021} and~\cite{DBLP:conf/ijcai/DoGN023}.

In the domain of multi-objective optimization, EDO was studied in~\cite{DBLP:conf/gecco/DoerrGN16}, where it was shown that on the \oneminmax benchmark problem the $(\mu + 1)$-DIVEA (which is very similar to the \gsemod in their setting) can compute a Pareto-optimal population with the best possible total Hamming distance in $O(n^3\log(n))$ expected iterations. It was conjectured in~\cite{DBLP:conf/foga/Antipov0N23} that this expected runtime is actually $O(n^2\log(n))$ when the algorithm starts with a population covering the entire Pareto front, and hence it is asymptotically not larger than the expected time needed for computing the whole Pareto front with classic multi-objective~EAs.

\subsection{Our contribution} 

In this paper, we aim at a better understanding of EDO processes in many-objective problems, and in particular, we address the following open questions.
\begin{enumerate}
    \item How can diversity be optimized in multi-objective problems, and what are the main sources of progress in its optimization?
    \item How does an EDO algorithm behave in situations where, in contrast to problems like \oneminmax, a significant portion of the space is \emph{not} on the Pareto front, and thus the front must first be located?
    \item Which diversity measure is most effective at guiding evolution towards populations of maximum diversity?
\end{enumerate}

We approach these questions by performing a rigorous analysis of EDO on multi-objective problems and for the first time consider a three-objective problem called $\lotz_k$ (formally defined in the following section). This problem is a modification of the classical bi-objective benchmark problem $\lotz$. The main difference is that all solutions that meet some quality threshold on the $\lotz$ function lie on the Pareto front of $\lotz_k$.

We perform a runtime analysis of a simple evolutionary multi-objective optimizer \gsemod, that optimizes diversity only when breaking ties between equal-fitness individuals competing for inclusion in the population. We prove that the \gsemod finds a Pareto-optimal population on $\lotz_k$ in $O(kn^3)$ expected iterations. We also prove, for two different diversity measures, upper bounds on the expected time for the \gsemod to find the optimal diversity starting from a Pareto-optimal population. For one measure (called the total imbalance) this bound is $O(kn^2\log(n))$, which is smaller than the upper bound on the runtime until finding a Pareto-optimal population. For the second diversity measure (called the sorted imbalances vector) we prove a weaker upper bound $O(k^2n^3\log(n))$. Our proofs are based on a rigorous study of which individuals are key to improving the diversity, and we believe that similar arguments might be fruitful in future theoretical studies of EDO.

We also show empirically that it takes the \gsemod a relatively short time to find the optimal diversity after covering the whole Pareto front for both diversity measures. This demonstrates the benefits of the EDO approach and also suggests that optimizing diversity via tie-breaking rules is an easy-to-implement and a very effective method. This is supported by additional experiments that demonstrate that the diversity is also effectively optimized even before the Pareto front is fully covered.

\section{Preliminaries}

In this paper we consider only pseudo-Boolean optimization, that is, our search space is the set of bit strings of length $n$. We use the following notation.
For integers $a,b \in \N$ with $a < b$, we denote by $[a..b]$ the integer interval $\{a, a + 1, \dots, b\}$.
Bits of a bit string $x$ of length $n$ are denoted by $x_i$, where $i \in [1..n]$. We assume that $x_1$ is the leftmost bit and $x_n$ is the rightmost bit.
For any non-negative integer $i$ by $1^i$ (or $0^i$) we denote the all-ones (or all-zeros) bit string of length $i$. If $i = 0$, this is an empty string. 
When we have a Boolean predicate $A$, we use Iverson bracket $[A]$ to map this Boolean value into $\{0, 1\}$. For a multi-objective function $f(x) = (f_1(x), f_2(x), \dots, f_m(x))$, we call $f(x)$ the \emph{fitness vector of $x$} and we call each of $f_1(x), \dots, f_m(x)$ the \emph{fitness values} of $x$.

\textbf{Dominance.} Given two points $x$ and $y$ and a $k$-objective function $f = (f_1, \dots, f_k)$ defined on these points, we say that $x$ dominates $y$ with respect to $f$, if for all $i \in [1..k]$ we have $f_i(x) \ge f_i(y)$ and there exists $j \in [1..k]$ for which $f_j(x) > f_j(y)$. We write it as $x \succ y$.

\subsection{The \gsemod}
\label{sec:gsemod}

The \emph{Simple Evolutionary Multi-objective Optimizer} (\semo) is an evolutionary algorithm for solving multi-objective problems. At all iterations this algorithm keeps a population of non-dominated solutions. The \semo is initialized with a point from the search space chosen uniformly at random. In each iteration it chooses an individual from its current population uniformly at random and mutates it. If the mutated offspring is not dominated by any individual in the population, it is added to the population, and all individuals which are dominated by this offspring are removed from the population. 

The \semo does not allow two individuals with the same fitness to be in the population. A situation when a new offspring $y$ is identical (in terms of fitness) to some individual $x$ in the population can be handled in different ways. Usually, $y$ replaces $x$ to enhance exploration of the search space, since $y$ is a new individual, and $x$ is an older one. This tie, however, can also be broken in a way that improves some secondary objective. In this paper we are interested in finding a \emph{diverse} set of non-dominated solutions, so we can remove an individual with which the diversity measure is worse. We call the \semo that optimizes a diversity measure $D$ in such way the \semod.

This mechanism is similar to the tie-breaking rule in the $(\mu + 1)$~genetic algorithm (GA) for single-objective optimization described by~\cite{DangFKKLOSS16}. There, it was shown that a diversity-improving tie-breaking rule allows to use crossover in a very effective way to escape local optima. This resulted in a $O(n\log n + 4^k)$ runtime on the $\jump_k$ benchmark, which is much smaller than the $O(n^{k - 1})$ upper bound on the runtime of most classic GAs on that problem~\cite{DangFKKLOSS18,DoerrEJK24}. In this paper, however, diversity is our primary goal, and the fitness has a role of a constraint (that is, we want the solutions to be Pareto-optimal). However, since diversity is a measure of a whole population, but not of a single individual, such a tie-breaking rule is a natural way to optimize diversity \emph{after} finding a Pareto-optimal population. 

In the literature, the \semo which uses standard bit mutation to generate new offspring is usually called the \emph{Global \semo} (\gsemo).  Similarly, we call the \semod with standard bit mutation the \gsemod. The pseudocode of the \gsemod is shown in Algorithm~\ref{alg:gsemod}. 

\begin{algorithm}[tp]
  Choose $x \in \{0,1\}^n$ uniformly at random\;
  $P\leftarrow \{x\}$\;

  \Repeat{$stop$}{
    Choose $x\in P$ uniformly at random\;
    Create $y$ by flipping each bit of $x$ with probability $\frac{1}{n}$\;
    \If{$\exists w \in P: f(w)=f(y)$}{
    \If{$D((P \cup \{y\}) \setminus \{w\})$ is not worse than $D(P)$}{$P \leftarrow (P  \cup \{y\}) \setminus \{w\}$\;}}
    \ElseIf{$\nexists w \in P: w \succ y$} {
      $P \leftarrow (P \cup \{y\})\backslash \{z\in P \mid y \succ z\};$
    }
  }
 \caption{The Global SEMO$_D$ maximizing a multi-objective function $f$ and optimizing a diversity measure $D$.} 
 \label{alg:gsemod}
 \end{algorithm}

\subsection{Diversity Measures}
\label{sec:diversity}

In this paper, we consider diversity measures based on the balance between 1-bits and 0-bits in each position in the population. For each $i \in [1..n]$ we denote by $n_1(i)$ the number of individuals in the population, which have a 1-bit in position $i$. More formally, $n_1(i) = \sum_{x \in P} x_i$, where $P$ is the population of the \gsemod. Similarly, by $n_0(i)$ we denote the number of 0-bits in position $i$. We define the \emph{imbalance} $b(i)$ of position $i$ as $|n_1(i) - n_0(i)|$. Intuitively, when we have a large imbalance in position $i$, the population is too homogeneous in that position, hence the optimal diversity implies minimizing the imbalance of each position.

Based on this observation, we define two diversity measures. The first one is called the \emph{total imbalance} and it is equal to the sum of the balances of all positions. Namely, we have $D(P) = \sum_{i = 1}^n b(i)$. The smaller this measure, the better the diversity.

The second measure is the \emph{sorted imbalances vector}, which is defined by vector $D(P) = (b(\sigma(i)))_{i = 1}^n$, where $\sigma$ is a permutation of positions $[1..n]$ such that the sequence $b(\sigma(i))_{i = 1}^n$ is non-increasing. When comparing two populations of the same size, the one with a lexicographically smaller vector $D(P)$ is more diverse. We do not determine how to compare diversity of populations of different sizes, since this comparison never occurs in the \gsemod, and most of the other classic EAs have populations of a fixed size during their whole run.

We note that using the imbalances to estimate the diversity is also implicitly used in the total Hamming distance, which was first shown in~\cite{DBLP:conf/gecco/WinebergO03a}. 

\subsection{$\lotz_k$ Problem}
\label{sec:lotz}

In this paper we consider a classic benchmark bi-objective function (\leadingones, \trailingzeros) (\lotz for brevity), which is defined on the space of bit strings of length $n$. We call $n$ the problem size. The first objective \leadingones (\lo for brevity) returns the length of the longest prefix consisting only of 1-bits, more formally,
\begin{align*}
    \leadingones(x) = \lo(x) = \sum_{i = 1}^n \prod_{j = 1}^i x_j.
\end{align*}
The second objective \trailingzeros (\tz for brevity) returns the length of the longest suffix which consists only of 0-bits, namely
\begin{align*}
    \trailingzeros(x) = \tz(x) = \sum_{i = 1}^n \prod_{j = n - i + 1}^n (1 - x_j).
\end{align*}
These two objectives contradict each other, and for any bit string $x$ we have $\lo(x) + \tz(x) \le n$.

The Pareto front of \lotz consists of $n + 1$ bit strings of form $1^i0^{n - i}$, for all $i \in [0..n]$. This means that the Pareto-optimal population has a fixed diversity. To study the aspects of diversity optimization by the \gsemod, we modify this problem. 

We introduce a parameter $k \in [2..n]$, and we say that all bit strings $x$ which have $\lo(x) + \tz(x) \ge n - k$ do not dominate each other. We call such bit strings and also their fitness vectors \emph{feasible}.
The fitness vectors of feasible bit strings are illustrated in Figure~\ref{fig:feasible-values}. 
We note that there is no bit string $x$ such that $\lo(x) + \tz(x) = n - 1$, for the following reason. Assume that $\lo(x) = i$. Then $x$ has prefix $1^i0$, hence $x_{i + 1} = 0$. If $\tz(x) = n - 1 - \lo(x) = n - 1 - i$, then $x$ also has suffix $10^{n - 1 - i}$. Then $x_{i + 1} = 1$, which contradicts the requirement on the prefix, hence we cannot have $\lo(x) + \tz(x) = n - 1$.

\begin{figure}
    \centering
    \scalebox{0.7}{
        \begin{tikzpicture}
            % fill the cells (this must be the lowest layer)
            \foreach \i [evaluate=\i as \jmin using {int(ifthenelse(\i < 9,20 - \i - 12,0))}, 
                         evaluate=\i as \jmax using 20-\i-4] in {0,...,16} {
                \foreach \j in {\jmin,...,\jmax} {
                    \draw [draw=none, fill=gray] (\i / 2, \j / 2) rectangle ++(0.5, 0.5);
                };
                \draw [draw=none, fill=gray] (\i / 2, 9 - \i / 2) rectangle ++(0.5, 0.5);
            };
            \draw [draw=none, fill=gray] (8.5, 0.5) rectangle ++(0.5, 0.5);
            \draw [draw=none, fill=gray] (9, 0) rectangle ++(0.5, 0.5);
        
            % Axes
            \draw [thick,->] (0, 0) -- (10, 0);
            \node [below] at (10, 0) {$\tz$};
            \draw [thick,->] (0, 0) -- (0, 10);
            \node [left] at (0, 10) {$\lo$};
    
            %drawing grid
            \foreach \i in {1,...,20} {
                \draw [thin] (0, \i/2) -- (10 - \i/2, \i/2);
                \draw [thin] (\i/2, 0) -- (\i/2, 10 - \i/2);
            };
    
            \node [left] at (0, 9.25) {$n$};
            \node [left] at (0, 8.75) {$n-1$};
            \node [left] at (0, 8.25) {$n-2$};
            \node [left] at (0, 4.25) {$n-k$};
            \node [left] at (0, 0.25) {$0$};
    
            \node [right, rotate=270] at (9.25, 0) {$n$};
            \node [right, rotate=270] at (8.75, 0) {$n-1$};
            \node [right, rotate=270] at (8.25, 0) {$n-2$};
            \node [right, rotate=270] at (4.25, 0) {$n-k$};
            \node [right, rotate=270] at (0.25, 0) {$0$};
        \end{tikzpicture}
    } % end scalebox
    \caption{Illustration of the feasible fitness vectors. Each square cell represents some fitness pair $(f_\lo, f_\tz)$. The feasible pairs, for which $f_\lo + f_\tz \ge n - k$ are shown in dark grey. Note that individuals with $f_\lo + f_\tz = n - 1$ do not exist.}
    \label{fig:feasible-values}
\end{figure}

To allow the \gsemod handle our requirement on the non-domination between feasible bit strings, we define $\lotz_k$ as a three-objective problem
\begin{align*}
    \lotz_k(x) = (\lo(x), \tz(x), h(\lo(x) + \tz(x))), 
\end{align*}  
where $h: \R \mapsto \R$ is defined as
\begin{align*}
    h(u) \coloneqq \begin{cases}
        0, & \text{ if } u < n - k, \\
        n + 1 - u, &\text{ if } u \ge n - k.
    \end{cases}
\end{align*}
From this definition it follows that for any $u, v \in [0..n]$ we have
\begin{align*}
    u > v \Rightarrow \begin{cases}
        h(u) \ge h(v), &\text{ if } v < n - k, \\
        h(u) < h(v), &\text{ if } v \ge n - k. \\
    \end{cases}
\end{align*}

Consequently, if $x$ dominates $y$ in terms of \lotz (that is, the first two objectives) and both of them are feasible, then we have $\lo(x) + \tz(x) > \lo(y) + \tz(y)$, and hence the third objective is better for $y$. Hence, any pair of feasible individuals do not dominate each other. Otherwise, if at least one of the two bit strings is not feasible, then the additional objective does not affect the domination relation. Since the third objective of $\lotz_k$ is determined by the first two objectives, we often omit it from the notation to simplify the reading, and for a bit string $x$ we write its fitness vector as $(\lo(x), \tz(x))$.

\textbf{Problem statement.} In this paper we study the behavior of the \gsemod on $\lotz_k$ for a variable parameter $k$. We aim to estimate the runtime of the \gsemod, that is, the number of iterations this algorithm performs until it finds a population which has solutions with all feasible fitness values in it, and also has the best possible diversity. We note that such a population cannot contain any infeasible solution, since each infeasible solution is dominated by at least one feasible solution, thus it cannot be accepted into the final population of the \gsemod. As a part of this problem, we also aim to estimate the time until the \gsemo finds a Pareto-optimal population of the three-objective $\lotz_k$ function.

\section{Optimal Diversity}
\label{sec:op-div-definition}
In this section, we show the best possible diversity of a population which covers the whole Pareto front of $\lotz_k$ (i.e., contains all feasible fitness vectors). Although we will not derive a simple formula for the optimal imbalances vector (that is, optimal imbalances for each position), we will show how the optimal diversity can be computed. We will use the observations from this section in our runtime analysis and also in our experiments to determine the moment when the algorithm finds the optimal diversity.

Before discussing the optimal diversity, we show the following lemma which estimates the population size of the \gsemod on different stages of its run.

\begin{lemma}
    \label{lem:pop-size}
    When the \gsemod runs on $\lotz_k$ with $k \ge 2$, before it finds the first feasible solution, the maximum population size is $\max_{x \in P}(\lo(x) + \tz(x)) + 1 \le n - k$. Once the \gsemod finds a feasible solution, the maximum population size is $\mu_{\max} \coloneqq nk - \frac{(k - 2)(k + 1)}{2} \le nk$. The size of any Pareto-optimal population containing all feasible fitness vectors is also $\mu_{\max}$.
\end{lemma}

\begin{proof}
    Before we find a feasible solution we can only have one individual per each \lo value in the population, since if we had two of them (and both were infeasible), then the one of them with a larger \tz value would dominate the other, which is impossible by the definition of the \gsemod. They also cannot have the same \tz value, since then they have an identical fitness, which is also impossible by the definition of the \gsemod. The number of different \lo values of the population is not greater than the maximum $\lo + \tz$ value in the population plus one. If we do not have feasible solutions in the population, this value is never greater than $n - k$.

    Once we have a feasible solution in the population, for each $\lo$ value $\ell < n$ we can have at most $\min(k, n - \ell)$ solutions with different $\tz$ values in the population, which are feasible solutions with this $\lo$ value (see Fig.~\ref{fig:feasible-values} for illustration). For $\ell = n$ we can have only one solution, which is $1^n$. Since we can have only one solution per each \lotz value, having more solutions for this \lo value would mean that we also have infeasible solutions in the population which are dominated by feasible ones (which contradicts our assumption on the Pareto-optimality of the population). Summing this up over all $\lo$ values, we obtain that the maximum population size $\mu_{\max}$ is at most
    \begin{align*}
        \mu_{\max} &= \sum_{\ell = 0}^{n - 1} \min(k, n - \ell) + 1 
        = k(n - k + 1) + \sum_{\ell = n - k + 1}^{n - 1} (n - \ell) + 1 \\
        &= nk - k^2 + k + \frac{k(k - 1)}{2} + 1 
        = nk - \frac{(k - 2)(k + 1)}{2} \le nk.
    \end{align*}

    This is the same as the total number of different feasible fitness vectors, thus the maximum size of a Pareto-optimal population. 
\end{proof}

We now estimate the minimum imbalance for each position. We aim to have individuals with all feasible fitness vectors in the population, and each fitness value implies particular bit values in some position of an individual with such fitness. Namely, an individual $x$ with $\lo(x) = j$ and $\tz(x) = \ell$ definitely has 1-bits in positions $[1..j] \cup \{n - \ell - 1\}$ and 0-bits in positions $\{j + 1\} \cup [n - \ell..n]$, however in positions $[j + 2..n - \ell - 2]$ any bit value is possible. We can distinguish the following three groups of fitness vectors for each position $i$. First, the feasible fitness vectors for which any individual with this fitness is guaranteed to have a 1-bit in position $i$. We denote the set of these values for position $i$ by $M_1(i)$ and then define $m_1(i) \coloneqq |M_1(i)|$. Second, the feasible fitness vectors which yield a 0-bit in position $i$, the set of which we denote by $M_0(i)$ and then define $m_0(i) \coloneqq |M_0(i)|$. All the remaining feasible fitness vectors, for which the bit value in position $i$ can be either one or zero, are denoted by $M(i)$ and then we define $m(i) \coloneqq |M(i)|$. With this notation, the following lemma establishes the minimum imbalance that can be reached in each position. 

\begin{lemma}\label{lem:imbalance}
    Consider a population, which covers all feasible fitness vectors of $\lotz_k$ with $k \ge 2$. For any position $i \in [1..n]$ the minimum imbalance is 
    \begin{align}\label{eq:opt-balance}
        b_{\textup{opt}}(i) = \max\left\{|m_0(i) - m_1(i)| -  m(i), \delta\right\},
    \end{align}
    where
    \begin{itemize}
        \item $m_0(i)$ is the number of individuals in the population which are guaranteed to have a 0-bit in position $i$, and 
        \begin{align*}
            m_0(i) &= \frac{i(i - 1)}{2} - [i > k + 1] \cdot \frac{(i - k - 1)(i - k)}{2} + \min\{n - i + 1, k\},
        \end{align*}
        \item $m_1(i)$ is the number of individuals in the population which are guaranteed to have a 1-bit in position $i$, and 
        \begin{align*}
            m_1(i) &= \frac{(n - i + 1)(n - i)}{2} - [i < n - k] \cdot \frac{(n - k - i)(n - k - i + 1)}{2} \\
            &+ \min\{i, k\},
        \end{align*}
        \item $m(i)$ is the number of individuals in the population which can have any value in position $i$, and 
        \begin{align*}
            m(i) &= \min\{k - 2, i - 1\} \cdot \min\{k - 2, n - i\} - \frac{a(i)(a(i) + 1)}{2}, \text{ where}\\
            a(i) &\coloneqq \max\{0, \min\{k - 3, i - 2, n - i - 1, n - k\}\},
        \end{align*}
        \item $\delta = \mu_{\max} \bmod 2$.
    \end{itemize}
\end{lemma}

\begin{proof}
    We first observe that if the difference between $m_1(i)$ and $m_0(i)$ is larger than $m(i)$, then the imbalance in this position will be at least $|m_0(i) - m_1(i)| -  m(i)$, which we have when all $m(i)$ free bits are set to the minority value. Otherwise, we can set the values of the $i$-th bits of individuals with fitness in $M(i)$ to any value and get the perfect imbalance in that position. This perfect imbalance $\delta$ is $1$, when the number of feasible individuals $\mu$ is odd, and it is $0$, if $\mu$ is even. This is equivalent to $\delta = \mu \bmod 2$ or, since by Lemma~\ref{lem:pop-size} the size of any population of all feasible fitness vectors is $\mu_{\max}$, it is also the same as $\delta = \mu_{\max} \bmod 2$. Since by Lemma~\ref{lem:pop-size} we have $\mu = \mu_{\max} = nk - \frac{(k + 2)(k - 1)}{2}$, we can compute the parity of $\mu$ depending on $n$ and $k$, which is shown in Table~\ref{tbl:parity}.
    Therefore, the smallest possible imbalance $\bopt(i)$ which we can have in position $i$ in a Pareto-optimal population is 
    \begin{align*}
        \bopt(i) = \max\left\{|m_0(i) - m_1(i)| -  m(i), \delta\right\}.
    \end{align*}
    
    \begin{table}[]
        \centering
        \begin{tabular}{c|c|c|c|c}
             \diagbox[innerwidth=2cm,innerleftsep=0.5cm,innerrightsep=-10pt,outerrightsep=10pt]{$n$}{$k \bmod 4$}& 0 & 1 & 2 & 3 \\ \hline
             odd & odd & even & even & odd \\ \hline
             even & odd & odd & even & even \\
        \end{tabular}
        \caption{The parity of the population size depending on the problem size $n$ and parameter $k$. Note that when $(k \bmod 4)$ is either $2$ or $3$, then the term $\frac{(k - 2)(k + 1)}{2}$ is even.}
        \label{tbl:parity}
    \end{table}
    
    We now fix an arbitrary $i$ and estimate the values of $m_0(i)$, $m_1(i)$ and $m(i)$. Each feasible fitness vector belongs to one (and only one) of three sets $M_1(i)$, $M_0(i)$ and $M(i)$. Consider some arbitrary feasible fitness vector $(f_\lo, f_\tz)$. An individual $x$ with this fitness must have zero in position $i$, if either $\tz(x) = f_\tz$ is at least $n - i + 1$, or if $\lo(x) = f_\lo$ is exactly $i - 1$. To belong to  $M_1(i)$, this fitness pair must also satisfy $f_\lo(x) + f_\tz(x) \ge n - k$. Similarly, $(f_\lo, f_\tz)$ belongs to $M_0(i)$, if $f_\lo \ge i$ or $f_\tz = n - i$, and $\lo(x) + \tz(x) \ge n - k$. The fitness vector $(f_\lo, f_\tz)$ belongs to $M(i)$ (that is, it does not determine the bit value in position $i$), if $f_\lo < i - 1$ and $f_\tz < n - i$ and $f_\lo + f_\tz \ge n - k$.
    We illustrate the fitness pairs which have a 1-bit, 0-bit or any bit in position $i$ in Figure~\ref{fig:guaranteed-bits}.
    
    \begin{figure}[t]
        \centering
        \scalebox{0.7}{
            \begin{tikzpicture}
                \foreach \i [evaluate=\i as \jmax using 20-\i-4] in {0,...,4} {
                    \foreach \j in {12,...,\jmax} {
                        \draw [draw=none, pattern={Lines[angle=0,line width=1pt,distance=2pt]}, pattern color=red] (\i / 2, \j / 2) rectangle ++(0.5, 0.5);
                    };
                    \draw [draw=none, pattern={Lines[angle=0,line width=1pt,distance=2pt]}, pattern color=red] (\i / 2, 9 - \i / 2) rectangle ++(0.5, 0.5);
                }
                \draw [draw=none, pattern={Lines[angle=0,line width=1pt,distance=2pt]}, pattern color=red] (3, 6) rectangle ++(0.5, 0.5);
                \draw [draw=none, pattern={Lines[angle=0,line width=1pt,distance=2pt]}, pattern color=red] (2.5, 6.5) rectangle ++(0.5, 0.5);
                \draw [draw=none, pattern={Lines[angle=0,line width=1pt,distance=2pt]}, pattern color=red] (3, 1) rectangle ++(0.5, 4.5);
                
                \foreach \i [evaluate=\i as \jmax using 20-\i-4,
                             evaluate=\i as \jmin using {int(ifthenelse(\i < 9,20 - \i - 12,0))}] in {7,...,16} {
                    \foreach \j in {\jmin,...,\jmax} {
                        \draw [draw=none, pattern={Lines[angle=90,line width=1pt,distance=2pt]}, pattern color=green!40!darkgray] (\i / 2, \j / 2) rectangle ++(0.5, 0.5);
                    };
                    \draw [draw=none, pattern={Lines[angle=90,line width=1pt,distance=2pt]}, pattern color=green!40!darkgray] (\i / 2 + 1, 8 - \i / 2) rectangle ++(0.5, 0.5);
                }
                \draw [draw=none, pattern={Lines[angle=90,line width=1pt,distance=2pt]}, pattern color=green!40!darkgray] (3.5, 5.5) rectangle ++(0.5, 0.5);
                \draw [draw=none, pattern={Lines[angle=90,line width=1pt,distance=2pt]}, pattern color=green!40!darkgray] (4, 5) rectangle ++(0.5, 0.5);
                \draw [draw=none, pattern={Lines[angle=90,line width=1pt,distance=2pt]}, pattern color=green!40!darkgray] (0, 5.5) rectangle ++(3, 0.5);
        
                \foreach \i [evaluate=\i as \jmin using 20-\i-12] in {0,...,5} {
                    \foreach \j in {\jmin,...,10} {
                        \draw [draw=none, pattern={Lines[angle=45,line width=1pt,distance=2pt]}, pattern color=blue] (\i / 2, \j / 2) rectangle ++(0.5, 0.5);
                    };
                }
            
                % Axes
                \draw [thick,->] (0, 0) -- (10, 0);
                \node [below] at (10, 0) {$\tz$};
                \draw [thick,->] (0, 0) -- (0, 10);
                \node [left] at (0, 10) {$\lo$};
        
                %drawing grid
                \foreach \i in {1,...,20} {
                    \draw [thin] (0, \i/2) -- (10 - \i/2, \i/2);
                    \draw [thin] (\i/2, 0) -- (\i/2, 10 - \i/2);
                };
        
                \node [left] at (0, 9.25) {$n$};
                \node [left] at (0, 8.75) {$n-1$};
                \node [left] at (0, 8.25) {$n-2$};
                \node [left] at (0, 6.25) {$i$};
                \node [left] at (0, 5.75) {$i - 1$};
                \node [left] at (0, 4.25) {$n-k$};
                \node [left] at (0, 0.25) {$0$};
        
                \node [right, rotate=270] at (9.25, 0) {$n$};
                \node [right, rotate=270] at (8.75, 0) {$n-1$};
                \node [right, rotate=270] at (8.25, 0) {$n-2$};
                \node [right, rotate=270] at (4.25, 0) {$n-k$};
                \node [right, rotate=270] at (3.75, 0) {$n - i + 1$};
                \node [right, rotate=270] at (3.25, 0) {$n - i$};
                \node [right, rotate=270] at (0.25, 0) {$0$};
            \end{tikzpicture}
        }
        \caption{Illustration of feasible fitness vectors and their effect on the bit value of some arbitrary position $i$. Each square cell represents some fitness pair $(f_\lo, f_\tz)$. The feasible pairs, for which $f_\lo + f_\tz \ge n - k$ are marked with stripes of different colors. Note that individuals with $f_\lo + f_\tz = n - 1$ do not exist. Horizontal red lines show the pairs which guarantee a 1-bit in position $i$. Vertical green lines show the pairs which guarantee a 0-bit in position $i$. Blue diagonal lines show the pairs, for which we can choose any bit value in position $i$.}
        \label{fig:guaranteed-bits}
    \end{figure}
    
    To compute $m_1(i)$, we split $M_1(i)$ into two disjoint sets $M_1'(i)$ and $M_1''(i)$ as follows.
    \begin{align*}
        M_1'(i) \coloneqq \{(f_\lo, f_\tz) :\xspace(f_\lo \ge i) \land (f_\lo + f_\tz \ge n - k) \},
    \end{align*}
    which is the red triangle (without a line of values for each $f_\lo + f_\tz = n - 1$) in Figure~\ref{fig:guaranteed-bits}, and
    \begin{align*}
        M_1''(i) \coloneqq \{(f_\lo, f_\tz) :\xspace(f_\tz = n - i) \land (f_\lo < i - 1)  \land (f_\lo + f_\tz \ge n - k) \},
    \end{align*}
    which is the red column in Figure~\ref{fig:guaranteed-bits}.
    The number of fitness pairs $(f_\lo, f_\tz)$ with $f_\lo \ge i$ (the red triangle in Figure~\ref{fig:guaranteed-bits}) is exactly $\frac{(n - i + 1)(n - i + 2)}{2} - (n - i)$. However, if $i < n - k$, then $\frac{(n - k - i)(n - k - i + 1)}{2}$ of them are infeasible. Therefore, we have
    \begin{align*}
        |M_1'(i)| &= \frac{(n - i + 1)(n - i + 2)}{2} - (n - i) \\
        &- [i < n - k] \cdot \frac{(n - k - i)(n - k - i + 1)}{2}.
    \end{align*}
    
    The number of fitness pairs $(f_\lo, f_\tz)$ with $f_\tz = n - i$ and $f_\lo < i - 1$ is $i - 1$. However, at most $k - 1$ of them are feasible. Hence,
    \begin{align*}
        |M_1''(i)| = \min\{i - 1, k - 1\} = \min\{i, k\} - 1.
    \end{align*}
    
    Consequently, we have
    \begin{align*}
        m_1(i) &= |M_1'(i)| + |M_1''(i)| \\
        &= \frac{(n - i + 1)(n - i + 2)}{2} - (n - i) \\
        &- [i < n - k] \cdot \frac{(n - k - i)(n - k - i + 1)}{2} \\
        &+ \min\{i, k\} - 1 \\
        & = \frac{(n - i + 1)(n - i)}{2} + \min\{i, k\} \\
        &- [i < n - k] \cdot \frac{(n - k - i)(n - k - i + 1)}{2}.
    \end{align*}
    
    With similar arguments we compute $m_0(i)$. We omit the details, since these calculations are symmetric to those for $m_1(i)$.
    \begin{align*}
        m_0(i) &= \frac{i(i - 1)}{2} - [i > k + 1] \cdot \frac{(i - k - 1)(i - k)}{2} \\
        &+ \min\{n - i + 1, k\}.
    \end{align*}
    
    To compute $m(i)$ we note that $M(i)$, which is highlighted in blue in Figure~\ref{fig:guaranteed-bits}, forms a rectangle with sides $\min\{k - 2, i - 1\}$ (along the $\lo$ axis) and $\min\{k - 2, n - i\}$ (along the \tz axis) with the cut out lower left corner of infeasible solutions. This corner is a triangle with a side which is by one smaller than the smallest side of the rectangle, but (when $k$ is large, namely $k > \frac{n}{2}$) it also cannot be larger than $n - k$. Therefore, defining this triangle's side as 
    \begin{align*}
        a(i) \coloneqq \max\{0, \min\{k - 3, i - 2, n - i - 1, n - k\}\},
    \end{align*}
    we obtain
    \begin{align}\label{eq:m_i}
        \begin{split}
            m(i) &= \min\{k - 2, i - 1\} \cdot \min\{k - 2, n - i\} \\
            &- \frac{a(i)(a(i) + 1)}{2}.
        \end{split}
    \end{align}
\end{proof}
    
Having the expressions for $m_1(i)$, $m_0(i)$ and $m(i)$ for all $i$, we can compute the optimal imbalances vector by putting these values into eq.~\eqref{eq:opt-balance}. In this way we compute the optimal diversity in our experiments to determine the moment when the diversity of a Pareto-optimal population becomes optimal.

The most important outcome of Lemma~\ref{lem:imbalance} for our theoretical investigation is that for each position $i$ to have an optimal imbalance $b(i) = \bopt(i)$ in this position in a Pareto-optimal population, we need to have a particular number of 1-bits in position $i$ among the $m(i)$ individuals from $M(i)$.

With this observation we say that the $i$-th bit of an individual of a Pareto-optimal population with fitness in $M(i)$ is \emph{wrong}, if it has a majority (over the whole population) value of the bits in this position. Otherwise it is \emph{right}. Note that we use this notation only for the bits of individuals with fitness in $M(i)$. The following two lemmas show, how many wrong bits we have in a position depending on its imbalance and how the wrong bits help us to improve the imbalance.

\begin{lemma}\label{lem:wrong-bits-number}
    For any position $i \in[1..n]$ the number of wrong bits in this position is at least $\frac{b(i) - \bopt(i)}{2}$. 
\end{lemma}

\begin{proof}
    If $b(i) = \bopt(i)$, then the lemma is trivial. If $b(i) > \bopt(i) \ge 0$, then the numbers of 1-bits and 0-bits in position $i$ are not equal. Without loss of generality, assume that we have more 1-bits in position $i$, and therefore, wrong bits are also 1-bits.
    Then, if we have $m'(i) \ge 0$ wrong bits, the imbalance of position $i$ is
    \begin{align*}
        b(i) &= |n_1(i) - n_0(i)| = n_1(i) - n_0(i) \\
        &= m_1(i) + m'(i) - m_0(i) - (m(i) - m'(i)) \\
        &= 2m'(i) + m_1(i) - m_0(i) - m(i) \\
        &\le 2m'(i) + |m_1(i) - m_0(i)| - m(i) \\
        &\le 2m'(i) + \bopt(i),
    \end{align*}
    where the last step used the definition of $\bopt(i)$ given in eq.~\eqref{eq:opt-balance}. Therefore,
    \begin{align*}
        2m'(i) \ge b(i) - \bopt(i).
    \end{align*}
    A similar argument can be used when we have more 0-bits than 1-bits in position $i$. 
\end{proof}

\begin{lemma}\label{lem:wrong-bits-flip}
    Consider some arbitrary $i \in [1..n]$ and a Pareto-optimal population of $\lotz_k$. If we have $b(i) > \bopt(i)$ and have $m'(i) > 0$ individuals in $M(i)$ with a wrong bit in position $i$, then replacing any of such individuals with the same bit string, but with a different bit value in position $i$ reduces $b(i)$ by two. The probability that the \gsemod reduces $b(i)$ by two in a single iteration is at least $\frac{m'(i)}{ekn^2}$.
\end{lemma}

\begin{proof}
    We first show that from condition $b(i) > \bopt(i)$ it follows that $b(i) \ge 2$. For odd population sizes we have $b(i) > \bopt(i) \ge 1$. For even population sizes $b(i)$ is always even (since the parity of $|n_1(i) - n_0(i)|$ is the same as the parity of $(n_1(i) + n_0(i))$), hence if $b(i) > 0$, then $b(i) \ge 2$.

    Without loss of generality, we assume that in position $i$ we have more 1-bits than 0-bits, thus the wrong bits are all ones. If we replace an individual with a wrong bit in position $i$ with an individual, which is different from it only in this position, we do not change the imbalance in any other position than $i$. In position $i$ we reduce the number of 1-bits and increase the number of 0-bits by one. Let $b'(i)$ be the imbalance of position $i$ after the replacement. Since $b(i) \ge 2$, then we have that $n_1(i) \ge n_0(i) + 2$. Hence after the replacement, we have
    \begin{align*}
        b'(i) &= |(n_1(i) - 1) - (n_0(i) + 1)| \\
        &= n_1(i) - n_0(i) - 2 = b(i) - 2.
    \end{align*}

    To generate such an individual, the \gsemod can choose one of $m'(i)$ individuals with a wrong bit in position $i$ as a parent and flip only the $i$-th bit in it. Since by Lemma~\ref{lem:pop-size} the population size is at most $nk$, the probability of this event is at least $\frac{m'(i)}{nk} \cdot \frac{1}{n}(1 - \frac{1}{n})^{n - 1} \ge \frac{m'(i)}{ekn^2}$. Note that if such an individual is created, it has exactly the same fitness as its parent (since the bits defining the \leadingones or \trailingzeros value have not been touched), and hence it competes with its parent for staying in the population. Since the offspring improves the imbalance of position $i$ without changing anything for other positions, it yields better diversity, and hence it wins against its parent and replaces it in the population.
\end{proof}

\section{Runtime Analysis of Covering All Fitness Vectors}
\label{sec:population-opt}

In this section we analyze the first stage of the algorithm, namely how it obtains a Pareto-optimal population which contains all feasible solutions. The main result of this section is the following theorem.
\begin{theorem}\label{thm:opt-pop-runtime}
    The expected runtime until the \gsemo or the \gsemod finds all feasible fitness values of $\lotz_k$ with $k \ge 2$ is $O(kn^3)$.
\end{theorem}

\begin{proof}
    We split the analysis into three phases. The first phase covers the steps from the initial population until we find a feasible solution. The second phase lasts until we have at least one feasible solution in the population for each \leadingones value. And the third phase lasts until we find all feasible solutions. 
    We note that once a feasible solution gets into a population of the \gsemo or of the \gsemod, its fitness vector will always be present in the population, since this solution is never dominated by any other solution.
    
    \textbf{Phase 1:} from the initial solution to a feasible solution. Let $X_t$ be $\max_{x \in P_t} (\lo(x) + \tz(x))$, where $P_t$ is the population in the beginning of iteration $t$. During this phase we have $X_t < n - k$ and the phase ends as soon as we get $X_t \ge n - k$. We also note that $X_t$ never decreases with $t$, since for any two bit strings $x$ and $y$ it is impossible for $y$ to dominate $x$ when $\lo(x) + \tz(x) > \lo(y) + \tz(y)$, hence any point $x$ can be removed from the population only by accepting a point with an equal or larger $\lo + \tz$ value. 
    
    To get $X_{t + 1} > X_t$ after iteration $t$ we can choose an individual $x$ with maximum value of $(\lo(x) + \tz(x))$ as a parent (or any such individual, if there are more than one) and increase either its $\lo$ value by flipping the first 0-bit in it (and not flipping any other bit) or its $\tz$ value by flipping its last 1-bit (and not flipping any other bit). Since we use standard bit mutation, the probability to flip only one of two particular bits is at least $\frac{2}{n}(1 - \frac{1}{n})^{n - 1} \ge \frac{2}{en}$. By Lemma~1, during this phase the population size in iteration $t$ is at most $X_t + 1$, hence the probability to choose such $x$ as a parent is at least $\frac{1}{X_t + 1}$. Therefore, the probability to increase $X_t$ is at least $\frac{2}{en(X_t + 1)}$. Hence, if $X_t = s$, then the expected number of iterations until we get a larger $X_t$ is $\frac{en(s + 1)}{2}$. For each $s$ we increase this value at most once, hence the expected time until we get a feasible solution is at most
    \begin{align*}
        \sum_{s = 0}^{n - k - 1} \frac{en(s + 1)}{2} = \frac{en}{2} \cdot \frac{(n - k)(n - k + 1)}{2} \le \frac{en^3}{4}.
    \end{align*}

    \textbf{Phase 2:} finding a feasible solution for each \lo value. In this phase we denote by $F_t \subset P_t$ the set of feasible solutions in the population and by $L_t = \{\lo(x)\}_{x \in F_t}$ the set of different \lo values which are present in $F_t$. We estimate the expected time $\tau$ until we find a feasible bit string $x$ with $\lo(x)$ not in $L_t$. We distinguish two cases.

    \emph{Case 1:} $\max\{L_t\} < n$. In this case, we can choose an individual $x$ from $F_t$ with the maximum \lo value as a parent and flip its first 0-bit. If the parent $x$ is the bit string of form $1^\ell0^{n - \ell}$, then this bit flip creates a new individual $x'$ with $\tz(x') = \tz(x) - 1$ and $\lo(x') = \lo(x) + 1$. Otherwise, this creates an individual $x'$, which has $\tz(x') = \tz(x)$ and $\lo(x') > \lo(x)$. In both scenarios, $x'$ is feasible and it adds a new $\lo$ value to $L_t$. The probability to chose such an individual with the maximum $\lo$ value as a parent is at least $\frac{1}{nk}$, since the population size is at most $nk$ by Lemma~1. The probability to flip only one particular bit is $\frac{1}{n}(1 - \frac{1}{n})^{n - 1} \ge \frac{1}{en}$. Hence, the probability to extend $L_t$ is at least $\frac{1}{ekn^2}$, and $\tau \le ekn^2$ in this case.

    \emph{Case 2:} $\max\{L_t\} = n$. In this case, we have at least one $\lo$ value $\ell \notin L_t$ for which we have $(\ell + 1) \in L_t$. Consider an individual $x \in F_t$ with $\lo(x) = \ell + 1$. If there are several such individuals let $x$ be the one with the largest $\tz(x)$ value. 
    
    If $\tz(x) > n - k - \lo(x)$, then choosing $x$ as a parent and flipping a 1-bit in position $\ell + 1$ results in an individual $x'$ with $\lo(x') = \ell$ and either $\tz(x') = \tz(x) + 1$ (if $x$ is of form $1^{\ell + 1}0^{n - \ell - 1}$) or $\tz(x') = \tz(x)$ (otherwise). In both ways we have $\lo(x') + \tz(x') \ge \lo(x) - 1 + \tz(x) \ge n-k$, hence $x'$ is feasible, and it adds $\ell$ to $L_t$. The probability to do that (similar to the previous case) is $\frac{1}{ekn^2}$, and therefore, $\tau \le ekn^2$.

    If $\tz(x) = n - k - \lo(x)$, then if we just reduce the $\tz$ value of $x$, we get an infeasible individual. However, we can choose $x$ as a parent and flip its last 1-bit. This gives us an individual $x'$ with $\lo(x') = \lo(x)$ and $\tz(x') \ge \tz(x) + 1$. Hence, this individual is feasible, and adding it into the population gives us an individual which satisfies $\tz(x') > n - k - \lo(x')$ and $\lo(x') = \ell + 1$, hence after obtaining it we will need at most $ekn^2$ expected iterations to add $\ell$ to $L_t$, as it has been shown in the previous paragraph. The probability to create such $x'$ is at least $\frac{1}{ekn^2}$, and expected time until it happens is at most $ekn^2$. Hence, in the worst scenario in Case 2 we have $\tau \le 2ekn^2$.

    To get all \lo values in $L_t$ we need to extend $L_t$ for at most $n$ times, hence the expected duration of Phase 2 is at most $2ekn^3$.
    
    \textbf{Phase 3:} covering all \tz values for each \lo value.
    Consider some arbitrary \lo value $\ell \in [0..n]$ and let $S$ be a set of individuals $x$ in the population with $\lo(x) = \ell$. There are two possible cases of how we can extend $S$ (that does not yet contain all possible \tz values), which depend on the maximum \tz value $s$ among all individuals in~$S$.
    
    \emph{Case 1:} $s < n - \ell$, then we can create a bit string with $\lo = \ell$ and $\tz > s$ by selecting the individual with $\lo = \ell$ and $\tz = s$ (the probability of this is at least $\frac{1}{nk}$) and flipping the last 1-bit in it without flipping any other bit (the probability of this is at least $\frac{1}{en}$). Hence, we create an individual with $\lo = \ell$ and with an unseen $\tz$ value with probability at least $\frac{1}{ekn^2}$. 

    \emph{Case 2:} $s = n - \ell$ (which is the maximum \tz value for \lo value $\ell$), then we can create a bit string $x$ with $\lo(x) = \ell$ and a new uncovered $\tz$ value $j < n - \ell - 1$ by selecting the individual with $\lo = \ell$ and $\tz = s$ and flipping a 0-bit in position $n - \ell$ in it. The probability of this is at least $\frac{1}{ekn^2}$

    We now consider $2ekn^3$ consecutive iterations. Let $k'$ be the number of uncovered \tz values that are missing in $S$ in the beginning of these iterations. Note that $k' < k \le n$. The event when we create a bit string with an uncovered \tz value for the fixed \lo value $\ell$ happens in each of the $2ekn^3$ iterations with probability at least $\frac{1}{ekn^2}$ until we cover all $k'$ uncovered \tz values. Hence, the number of such events during these iterations dominates a random variable $X \sim \min(k', Y)$, where $Y \sim \Bin(2ekn^3, \frac{1}{ekn^2}))$. During this series of iterations we have less than $k'$ such events with probability at most $\Pr[X \le k'] = \Pr[Y \le k'] \le \Pr[Y \le n]$. 
    Since $E[Y] = 2n$, by a Chernoff bound the latter probability is at most
    \begin{align*}
        \Pr\left[Y \le \left(1 - \frac{1}{2}\right) E[Y]\right] \le \exp\left(-\frac{\frac{1}{4}E[Y]}{3}\right) = e^{-n/6}.
    \end{align*}
    By the union bound over all $(n - 1)$ different \lo values\footnote{We exclude values $n$ and $n - 1$, for which we have only one feasible pair and hence for those \lo values all \tz values are covered after the second phase.} the probability that after $2ekn^3$ iteration we have at least one such value with a non-covered \tz value is at most $(n - 1)e^{-n/12}= o(1)$. Hence, the expected number of such phases of length $2ekn^3$ which it takes to cover all feasible solutions is $(1 + o(1))$. Therefore, the expected duration of Phase~3 is at most $2e(1 + o(1))kn^3$.

    Summing up the expected times for each phase, we obtain that the expected runtime is at most
        \begin{align*}
        \frac{en^3}{4} + 2ekn^3 + 2e(1 + o(1))kn^3 = O(kn^3).
    \end{align*}

    Note that the tie-breaking rule does not play any role in this proof, hence the result holds both for the \gsemo and for the \gsemod.
\end{proof}

\section{Runtime Analysis of Diversity Optimization}
\label{sec:diversity-opt}

In this section we analyze how much time it takes the \gsemod to find a population with optimal diversity after it has already found a population of all feasible solutions. We  start with the following theorem for the total imbalance measure.

\begin{theorem}\label{thm:total-balance}
    Consider a run of the \gsemod on $\lotz_k$ with $k \ge 2$ minimizing the diversity measure $D(P) = \sum_{i = 1}^n b(i)$ and starting with a population $P_0$ that covers all feasible fitness vectors. Then the expected runtime of the \gsemod until it finds a population with the best possible diversity is $O(kn^2\ln(n))$.
\end{theorem}
\begin{proof}
    Let $P_t$ be a population of the \gsemod in the beginning of iteration $t$ and let $\phi_t(i)$ be the difference between the imbalance of position $i$ and its optimal imbalance in $P_t$, that is, $\phi_t(i) \coloneqq b(i) - \bopt(i)$. Let also $\Phi_t \coloneqq \sum_{i = 1}^n \phi_t(i)$, which we call the \emph{potential} of population in iteration $t$. 
    Note that the potential decreases strongly monotonically with the diversity $D(P_t)$ and hence no population increasing the potential is accepted.
    When $\Phi_t = 0$, it implies that all $\phi_t(i) = 0$, and therefore, population $P_t$ has an optimal diversity. Therefore, to estimate the runtime of the \gsemod, we need to estimate the time until $\Phi_t$ becomes zero.

    Note that for each $i$ the imbalance of position $i$ is defined by the $i$-th bits of individuals with fitness in $M(i)$, hence the maximum difference of $b(i)$ and $\bopt(i)$ is $m(i)$ which by Lemma~\ref{lem:imbalance} is at most $k^2$. Therefore, each $\phi_t(i)$ is at most $k^2$ and thus, $\Phi_t$ is at most $nk^2$. 

    By Lemma~\ref{lem:wrong-bits-flip}, the probability to reduce $\phi_t(i)$ by two in one iteration is at least $\frac{m'(i)}{ekn^2}$, where we recall that $m'(i)$ is the number of wrong bits in position $i$. By Lemma~\ref{lem:wrong-bits-number} we have $m'(i) \ge \frac{\phi_t(i)}{2}$, hence the probability to reduce $\phi_t(i)$ by two is at least $\frac{\phi_t(i)}{2ekn^2}$. The probability to reduce $\Phi_t$ by two is at least the probability that we reduce at least one $\phi_t(i)$ by two. Since the events considered in Lemma~\ref{lem:wrong-bits-flip} are disjoint for different positions, we have
    \begin{align*}
        \Pr[\Phi_t - \Phi_{t + 1} = 2] \ge \sum_{i = 1}^n \frac{\phi_t(i)}{2ekn^2} = \frac{\Phi_t}{2ekn^2}.
    \end{align*}
    For each value $\Phi_t$ we reduce this value at most once, since $\Phi_t$ does not increase. Conditional on $\Phi_t = s$, the probability to reduce it is at least $\frac{s}{2ekn^2}$, and the expected time until we reduce $\Phi_t$ is at most $\frac{2ekn^2}{s}$. Since $\Phi_t$ can only take integer values from $[1..nk^2]$ before we find the optimum, the total expected runtime until we find the optimal population is at most the sum of the times to reduce each of the possible values of $\Phi_t$, that is,
    \begin{align*}
        E[T] &\le \sum_{s = 1}^{nk^2} \frac{2ekn^2}{s} \le 2ekn^2 (\ln(nk^2) + 1) \\
        &\le 2ekn^2 (3\ln(n) + 1) = O(kn^2\ln(n)). 
    \end{align*} 
\end{proof}

Note that Theorem~\ref{thm:total-balance} gives an upper bound which is asymptotically smaller than the upper bound on the runtime until the \gsemod finds all feasible solutions, which by Theorem~\ref{thm:opt-pop-runtime} is $O(kn^3)$, hence the expected runtime until the \gsemod finds a population of all feasible solutions with an optimal diversity starting from a random bit string is also $O(kn^3)$.

The second diversity measure which we consider in this paper is the sorted imbalances vector which is to be minimized lexicographically. For this diversity measure we show the following theorem.

\begin{theorem}\label{thm:sorted-balances}
    Consider a run of the \gsemod on $\lotz_k$ with $k \ge 2$ starting with a population $P_0$ that covers all feasible fitness vectors, and is minimizing the diversity measure $D(P) = (b(\sigma(i)))_{i = 1}^n$, where $\sigma$ is a permutation of positions $[1..n]$ such that sequence $(b(\sigma(i)))_{i = 1}^n$ is non-increasing. Then the expected runtime until the \gsemod finds a population with the best possible diversity is $O(k^2n^3\log(n))$.
\end{theorem}

\begin{proof}
    Let $P_t$ be the population of the \gsemod in the beginning of iteration $t$ and let $\Phi_t = \max_{i : b(i) > \bopt(i)} b(i)$, that is, the first element of $D(P_t)$ that can be reduced. Note that $\Phi_t$ never increases, since this would result into a lexicographically larger sorted imbalances vector. Let also $\Psi_t$ be the number of positions with imbalance $\Phi_t$ in iteration $t$, the imbalance of which can be reduced. If at some iteration $t$ we have $\Phi_t = s$, then $\Psi_t$ does not increase until $\Phi_t$ is reduced, otherwise we would get a lexicographically larger imbalances vector. Hence if we define the potential as a pair $(\Phi_t, \Psi_t)$, then it never increases in lexicographical sense.

    We find the optimum as soon as we reach the minimum potential, when $\Phi_t = \delta$ (which is either one or zero depending on the population size parity) and $\Psi_t = 0$. The maximum value of $\Phi_t$ is the maximum imbalance of any position, which is at most the population size, hence $\Phi_0 \le nk$. The maximum value of $\Psi_t$ is $n$, since we have $n$ positions in total.
    
    To decrease $\Phi_t$, we can first decrease $\Psi_t$ to one and then decrease the imbalance of the only position $i$ which has imbalance $b(i) = \Phi_t > \bopt(i)$. To decrease $\Psi_t$ by one, we can decrease the imbalance of any of $\Psi_t$ positions which have the maximum non-optimal imbalance. Consider one of these positions $i$. By Lemma~\ref{lem:wrong-bits-flip}, the probability to decrease its imbalance is at least $\frac{1}{ekn^2}$, since by Lemma~\ref{lem:wrong-bits-number} there is at least one individual with a wrong bit in this position. The probability that we decrease the imbalance of any of $\Psi_t$ positions is therefore at least $\frac{\Psi_t}{ekn^2}$. Hence, if $\Psi_t = s$, then the expected number of iterations until we decrease it by one is at most $\frac{ekn^2}{s}$. When $\Psi_t = 1$, the expected time until we decrease the imbalance of this position is $ekn^2$. Hence, summing up the expected times to decrease $\Psi_t$ for all possible values of it we get an upper bound on the expectation of time $T'$ until we decrease the value of $\Phi_t$.
    \begin{align*}
        E[T'] \le \sum_{s = 1}^{n} \frac{ekn^2}{s} \le ekn^2(\ln(n) + 1).
    \end{align*}
    
    We can reduce the value of $\Phi_t$ at most $nk$ times until we reach the optimal diversity, hence the total expected runtime until we find a population with the optimal diversity is at most 
    \begin{align*}
        E[T] \le knE[T'] = ek^2n^3(\ln(n) + 1) = O(k^2n^3\log(n)).
    \end{align*} 
\end{proof}

We note that when we optimize the diversity measured by the sorted imbalances vector, the total imbalance might increase. This happens when we decrease the imbalance of some position $i$, but also increase imbalances of positions which at the moment have smaller imbalances than position $i$. This situation resembles optimization of the \binval benchmark function with the \oea, for which the \onemax value can decrease, but it does not slow down the optimization~\cite{DBLP:journals/cpc/Witt13}. This analogy makes us optimistic that it takes much less time to optimize this sorted imbalances vector than our upper bound in Theorem~\ref{thm:sorted-balances}, and the results of the experiments shown in the next section support this optimism. However, the main problem with proving it is that imbalances can be changed in large chunks, when we replace one individual with a one-bit mutation of another, but not of itself.

\section{Experiments}
\label{sec:experiments}

In this section we report the results of our empirical study. This study includes three experiments. First, we ran the \gsemod on $\lotz_k$ with varying parameters $n$ and $k$ and using different diversity measures and tracked the time until it found a population covering the whole Pareto front and also the time it further took the algorithm to find a population with optimal diversity. Then, we ran the \gsemo and the \gsemod on $\lotz_k$ (with the same set of parameters $n$ and $k$) until they covered the whole Pareto front and measure the diversity of the resulting population. In the last experiment, we ran the \gsemod on $\lotz_k$ starting from a population that covers the whole Pareto front, but has the worst possible diversity. 

\subsection{Experiments with Total Runtime}
\label{sec:experiments-total-runtime}

In this subsection we discuss an experiment where we ran the \gsemod on $\lotz_k$ on different problem sizes and with different values of $k$. We used $n \in \{2^3, 2^4, 2^5, 2^6, 2^7\}$ and for each $n$, except $n=2^7$, we used $k \in \{2, 4, \lfloor\sqrt{n}\rfloor, \frac{n}{2}, n\}$. 
For the largest $n = 2^7$ we used only $k \in \{2, 4, \lfloor\sqrt{n}\rfloor\}$. We used both diversity measures (the total imbalance and the sorted imbalances vector) and made $128$ runs of the \gsemod for each parameter setting and each diversity measure.

The results are shown in Figure~\ref{fig:total-runtimes}. The plots show the mean runtimes over $128$ runs and include error bars indicating the standard deviation. The dashed lines correspond to the runtimes until the entire Pareto front is covered, and solid lines correspond the total runtimes until a population with maximum diversity is found. 
All of them are normalized by the upper bound from Theorem~\ref{thm:opt-pop-runtime} on the time of the first phase of the optimization, that is, by $kn^3$, which helps to see the differences between the plots.

\begin{figure}[t]
    \begin{center}
        \begin{tikzpicture}
            \begin{axis}[width=0.5\linewidth, height=0.28\textheight,
                at={(0, 0)},anchor=outer south west,
                cycle list name=tikzcycle, grid=major,  xmode=log, log base x=2, legend style={at={(1.1, 1.2)}, anchor=south,legend columns = 5, inner sep=5pt},
                title = {Total imbalance},
                ymin = 0,
                legend cell align={left},
                xlabel={Problem size $n$}, ylabel={Runtime $/kn^3$}]
            
            \addplot plot [error bars/.cd, y dir=both, y explicit] coordinates
                {(8,0.43900299072265625)+-(0,0.16532317777627983)(16,0.43033504486083984)+-(0,0.13937613939308294)(32,0.36084556579589844)+-(0,0.1271515427038478)(64,0.289188951253891)+-(0,0.09082381296126739)(128,0.2322248313575983)+-(0,0.07361368690754483)};
            \addlegendentry{$k = 2$};
            \addplot plot [error bars/.cd, y dir=both, y explicit] coordinates
                {(8,0.2943267822265625)+-(0,0.08225170524684278)(16,0.2648506164550781)+-(0,0.05980208363893475)(32,0.23014366626739502)+-(0,0.05145233770685326)(64,0.18716266751289368)+-(0,0.04466418563898328)(128,0.143635175190866)+-(0,0.03770990634673649)};
            \addlegendentry{$k = 4$};
            \addplot plot [error bars/.cd, y dir=both, y explicit] coordinates
                {(8,0.43900299072265625)+-(0,0.16532317777627983)(16,0.2648506164550781)+-(0,0.05980208363893475)(32,0.20164270401000978)+-(0,0.04522278473263609)(64,0.13222819566726685)+-(0,0.02397940401689726)(128,0.08838175474242731)+-(0,0.016372777418544628)};
            \addlegendentry{$k = \sqrt{n}$};
            \addplot plot [error bars/.cd, y dir=both, y explicit] coordinates
                {(8,0.2943267822265625)+-(0,0.08225170524684278)(16,0.20325803756713867)+-(0,0.04116444392557903)(32,0.14370539784431458)+-(0,0.027017532406237624)(64,0.12993510533124208)+-(0,0.022639246920744743)};
            \addlegendentry{$k = n/2$};
            \addplot plot [error bars/.cd, y dir=both, y explicit] coordinates
                {(8,0.1787891387939453)+-(0,0.059747471059267976)(16,0.14882993698120117)+-(0,0.027150873769234647)(32,0.12385135889053345)+-(0,0.020413799171951637)(64,0.1295779631473124)+-(0,0.01843546563174857)};
            \addlegendentry{$k = n$};
            \addplot plot [error bars/.cd, y dir=both, y explicit] coordinates
                {(8,0.43900299072265625)+-(0,0.16532317777627983)(16,0.43033504486083984)+-(0,0.13937613939308294)(32,0.36084556579589844)+-(0,0.1271515427038478)(64,0.289188951253891)+-(0,0.09082381296126739)(128,0.2322248313575983)+-(0,0.07361368690754483)};
            \addplot plot [error bars/.cd, y dir=both, y explicit] coordinates
                {(8,0.2450714111328125)+-(0,0.07586462304979455)(16,0.22418212890625)+-(0,0.05604200043054123)(32,0.2041972279548645)+-(0,0.05321852538273462)(64,0.17365621775388718)+-(0,0.04512927605581757)(128,0.13567626848816872)+-(0,0.03865712382793666)};
            \addplot plot [error bars/.cd, y dir=both, y explicit] coordinates
                {(8,0.43900299072265625)+-(0,0.16532317777627983)(16,0.22418212890625)+-(0,0.05604200043054123)(32,0.17589468955993653)+-(0,0.04418025155138036)(64,0.11029953509569168)+-(0,0.023915610083788532)(128,0.06981660527261821)+-(0,0.012466362953836484)};
            \addplot plot [error bars/.cd, y dir=both, y explicit] coordinates
                {(8,0.2450714111328125)+-(0,0.07586462304979455)(16,0.15094876289367676)+-(0,0.03735071023529397)(32,0.10785552859306335)+-(0,0.026200846600058607)(64,0.08426874782890081)+-(0,0.017701663849637762)};
            \addplot plot [error bars/.cd, y dir=both, y explicit] coordinates
                {(8,0.15321731567382812)+-(0,0.05813841814677055)(16,0.11417484283447266)+-(0,0.029328979330548944)(32,0.09122972935438156)+-(0,0.021678150364653873)(64,0.07165798963978887)+-(0,0.011549565048606908)};
            \end{axis}

            \begin{axis}[width=0.5\linewidth, height=0.28\textheight,
                at={(0.5\linewidth, 0)},anchor=outer south west,
                cycle list name=tikzcycle, grid=major,  xmode=log, log base x=2, ymin = 0,
                title = {Sorted imbalances vector},
                xlabel={Problem size $n$}, ylabel={Runtime $/kn^3$}]
            
            \addplot plot [error bars/.cd, y dir=both, y explicit] coordinates
                {(8,0.4182281494140625)+-(0,0.15868109963032725)(16,0.4122171401977539)+-(0,0.13150247559219425)(32,0.34922707080841064)+-(0,0.10167566117996371)(64,0.296432688832283)+-(0,0.09767318222618054)(128,0.2223012112081051)+-(0,0.06222242029630154)};
            % \addlegendentry{$k = 2$};
            \addplot plot [error bars/.cd, y dir=both, y explicit] coordinates
                {(8,0.2772979736328125)+-(0,0.0795570829753459)(16,0.26013803482055664)+-(0,0.06504844006881105)(32,0.23202568292617798)+-(0,0.045090343683486106)(64,0.18096600472927094)+-(0,0.04556401159660343)(128,0.1413938170298934)+-(0,0.03755383025351785)};
            % \addlegendentry{$k = 4$};
            \addplot plot [error bars/.cd, y dir=both, y explicit] coordinates
                {(8,0.4182281494140625)+-(0,0.15868109963032725)(16,0.26013803482055664)+-(0,0.06504844006881105)(32,0.20127525329589843)+-(0,0.04402523425639102)(64,0.12697172164916992)+-(0,0.02416528078057575)(128,0.08626481281085448)+-(0,0.014389380643827259)};
            % \addlegendentry{$k = \sqrt{n}$};
            \addplot plot [error bars/.cd, y dir=both, y explicit] coordinates
                {(8,0.2772979736328125)+-(0,0.0795570829753459)(16,0.19586849212646484)+-(0,0.04045857110594427)(32,0.13345037400722504)+-(0,0.022180955830307388)(64,0.10867930203676224)+-(0,0.02141590597788401)};
            % \addlegendentry{$k = n/2$};
            \addplot plot [error bars/.cd, y dir=both, y explicit] coordinates
                {(8,0.17180442810058594)+-(0,0.053559542062586146)(16,0.13403761386871338)+-(0,0.0281321061795537)(32,0.09580576419830322)+-(0,0.014301808131540455)(64,0.07590017886832356)+-(0,0.010684738042759576)};
            % \addlegendentry{$k = n$};
            \addplot plot [error bars/.cd, y dir=both, y explicit] coordinates
                {(8,0.4182281494140625)+-(0,0.15868109963032725)(16,0.4122171401977539)+-(0,0.13150247559219425)(32,0.34922707080841064)+-(0,0.10167566117996371)(64,0.296432688832283)+-(0,0.09767318222618054)(128,0.2223012112081051)+-(0,0.06222242029630154)};
            \addplot plot [error bars/.cd, y dir=both, y explicit] coordinates
                {(8,0.22883987426757812)+-(0,0.07236703709377393)(16,0.21717548370361328)+-(0,0.06505781392375447)(32,0.20594894886016846)+-(0,0.04549436602443897)(64,0.16676345467567444)+-(0,0.044040900688731066)(128,0.13498032931238413)+-(0,0.03871409578858694)};
            \addplot plot [error bars/.cd, y dir=both, y explicit] coordinates
                {(8,0.4182281494140625)+-(0,0.15868109963032725)(16,0.21717548370361328)+-(0,0.06505781392375447)(32,0.1746135711669922)+-(0,0.0418903833377037)(64,0.10604174062609673)+-(0,0.022231842367924557)(128,0.07401483743028207)+-(0,0.014066128338292534)};
            \addplot plot [error bars/.cd, y dir=both, y explicit] coordinates
                {(8,0.22883987426757812)+-(0,0.07236703709377393)(16,0.14924025535583496)+-(0,0.03889224727760326)(32,0.10941889882087708)+-(0,0.025159718487839256)(64,0.08819631207734346)+-(0,0.017313762877203046)};
            \addplot plot [error bars/.cd, y dir=both, y explicit] coordinates
                {(8,0.15659141540527344)+-(0,0.05331780162114653)(16,0.11610889434814453)+-(0,0.02804786780215235)(32,0.0890117660164833)+-(0,0.01650239905537735)(64,0.07474494073539972)+-(0,0.011510666179926741)};
            \end{axis}
        \end{tikzpicture}
    \end{center}
    \caption{The normalized runtimes of the \gsemod when optimizing the total imbalance diversity (on the left) and the sorted imbalances vector (on the right). The dashed lines show the time until a Pareto-optimal population is found and the solid lines show the time until the optimal diversity is reached. All runtimes are normalized by $kn^3$, which is asymptotically the same as the upper bound shown in Theorem~\ref{thm:opt-pop-runtime}.}
    \label{fig:total-runtimes}
\end{figure}

The results of the runs when the \gsemod minimized the sum of total imbalances are shown in the left plot in Figure~\ref{fig:total-runtimes}. In this figure we see that all the normalized runtimes (both the runtimes until we obtain a Pareto-optimal population, indicated by the dashed lines, and runtimes until the optimal diversity, indicated by the solid lines) are decreasing, which suggests that the asymptotical upper bound might be even smaller than $O(kn^3)$, but not by a large factor.

We observe that the time required by the algorithm to get an optimal diversity after computing a Pareto-optimal population is small compared to the time required to find a Pareto-optimal population for the first time.
This matches the ratio between the upper bounds shown in Theorems~\ref{thm:opt-pop-runtime} and~\ref{thm:sorted-balances}, however, we note that without a proof of a matching lower bound for the first stage (until we cover the Pareto front) we cannot state that the duration of the second stage (that is, after finding a Pareto-optimal population) is a small fraction of the total runtime.

Also, note that for $k = 2$ both lines coincide. This is because for any fitness pair $(f_\lo, f_\tz)$ such that $f_\lo + f_\tz = n - 2$ there exists only one bit string with this fitness, hence there exists a unique Pareto-optimal population, which therefore has an optimal diversity.

The results of the runs when the \gsemod minimized the imbalances vector, sorted in descending order, are shown in the right plot in Figure~\ref{fig:total-runtimes}. In this figure we see that all the normalized runtimes are decreasing, which suggests that the asymptotical upper bound might be even smaller than $O(kn^3)$, but not by a large factor. This also indicates that the upper bound on the diversity optimization time given in Theorem~\ref{thm:sorted-balances} is not tight and in practice the time required for diversity optimization is not larger than the time needed for finding a Pareto-optimal population. For large values of $k$ the duration of diversity optimization is even smaller than it is for the total imbalance.

\subsection{Diversity after Covering the Pareto Front}

In this subsection, we discuss an experiment, where we ran the \gsemo and the \gsemod on $\lotz_k$ from a random solution and until they covered the whole Pareto front. We used the same values of $n$ and $k$ as for the experiments described in Section~\ref{sec:experiments-total-runtime}. We ran the \gsemod with each of the two diversity measures (the total imbalance and the sorted imbalances vector) and made $128$ runs for each parameter setting and each diversity measure. For each run we remembered the total imbalance at the moment when the algorithm covered the whole Pareto front. Figure~\ref{fig:total-imbalance-experiment} illustrates the mean value of the total imbalance and its standard deviation for all three approaches depending on the problem size. Solid lines correspond to the total imbalance measure, dashed lines correspond to the sorted imbalances vector, and the dashed-dotted lines correspond to no diversity optimization. To show the difference between the three approaches in more details, we apply the following transformations to the resulting total imbalance of each run (in this order).

\begin{enumerate}
    \item The values of total imbalance are normalized by the optimal imbalance for this problem size $n$ and parameter $k$. 
    \item From the normalized values we subtract one. 
    \item We use logarithmic scale of $Y$-axis.
\end{enumerate}

The first step allows to disregard the difference in the optimal total imbalance for different problem sizes and different values of $k$, which makes the comparison fair. However, it results in values that are all only slightly larger than one. The second and the third steps allow to amplify the visual difference between these values. Note that we omit $k=2$ (including the point for $n = 8$ and $k = \lfloor\sqrt{n}\rfloor = 2$), since for this $k$ we always have the optimal diversity once we cover the Pareto front.

\begin{figure}[t]
    \begin{center}
        \begin{tikzpicture}
            \begin{axis}[width=0.95\linewidth, height=0.28\textheight,
                at={(0, 0)},anchor=outer south west,
                cycle list name=tikzcycle-notwo, grid=major,  xmode=log, log base x=2, 
                ymode=log, log base y=10, ylabel style={align=center},
                legend style={at={(0.5, 1.2)}, anchor=south,legend columns = 5, inner sep=5pt},
                legend cell align={left},
                xlabel={Problem size $n$}, ylabel={Normalized total imbalance\\minus one}]
                \addplot plot [error bars/.cd, y dir=both, y explicit] coordinates{(8.0,0.0302734375)+-(0,0.0344675379)(16.0,0.0071732955)+-(0,0.0070876981)(32.0,0.0015479123)+-(0,0.0016038847)(64.0,0.0004642835)+-(0,0.0004499696)(128.0,0.0001018812)+-(0,0.0001052467)};
                \addlegendentry{$k = 4$};
                \addplot plot [error bars/.cd, y dir=both, y explicit] coordinates{(16.0,0.0071732955)+-(0,0.0070876981)(32.0,0.0020368782)+-(0,0.0022178560)(64.0,0.0007542959)+-(0,0.0007055796)(128.0,0.0002406158)+-(0,0.0001788462)};
                \addlegendentry{$k = \sqrt{n}$};
                \addplot plot [error bars/.cd, y dir=both, y explicit] coordinates{(8.0,0.0302734375)+-(0,0.0344675379)(16.0,0.0138658940)+-(0,0.0132886157)(32.0,0.0035713338)+-(0,0.0032712504)(64.0,0.0007920681)+-(0,0.0005651316)};
                \addlegendentry{$k = n/2$};
                \addplot plot [error bars/.cd, y dir=both, y explicit] coordinates{(8.0,0.0417480469)+-(0,0.0426617842)(16.0,0.0130888526)+-(0,0.0126707479)(32.0,0.0042826705)+-(0,0.0034277940)(64.0,0.0024994399)+-(0,0.0010233301)};
                \addlegendentry{$k = n$};
                \addplot plot [error bars/.cd, y dir=both, y explicit] coordinates{(8.0,0.0296875000)+-(0,0.0330349602)(16.0,0.0069602273)+-(0,0.0066488805)(32.0,0.0017991442)+-(0,0.0018222139)(64.0,0.0004407755)+-(0,0.0004190570)(128.0,0.0000960870)+-(0,0.0000893662)};
                \addplot plot [error bars/.cd, y dir=both, y explicit] coordinates{(16.0,0.0069602273)+-(0,0.0066488805)(32.0,0.0020033769)+-(0,0.0021475030)(64.0,0.0006999699)+-(0,0.0006481477)(128.0,0.0001931234)+-(0,0.0001773373)};
                \addplot plot [error bars/.cd, y dir=both, y explicit] coordinates{(8.0,0.0296875000)+-(0,0.0330349602)(16.0,0.0102442053)+-(0,0.0087552643)(32.0,0.0020725011)+-(0,0.0022897970)(64.0,0.0001488535)+-(0,0.0001745878)};
                \addplot plot [error bars/.cd, y dir=both, y explicit] coordinates{(8.0,0.0310058594)+-(0,0.0385703556)(16.0,0.0054221082)+-(0,0.0070251174)(32.0,0.0009055398)+-(0,0.0014379559)(64.0,0.0000319431)+-(0,0.0000784899)};
                \addplot plot [error bars/.cd, y dir=both, y explicit] coordinates{(8.0,0.2105468750)+-(0,0.0529048172)(16.0,0.0915838068)+-(0,0.0142282739)(32.0,0.0460240534)+-(0,0.0046407235)(64.0,0.0230966337)+-(0,0.0018093502)(128.0,0.0117076445)+-(0,0.0006657773)};
                \addplot plot [error bars/.cd, y dir=both, y explicit] coordinates{(16.0,0.0915838068)+-(0,0.0142282739)(32.0,0.0732539130)+-(0,0.0065653462)(64.0,0.0814232164)+-(0,0.0031374625)(128.0,0.0640262499)+-(0,0.0014202032)};
                \addplot plot [error bars/.cd, y dir=both, y explicit] coordinates{(8.0,0.2105468750)+-(0,0.0529048172)(16.0,0.3453280215)+-(0,0.0311838758)(32.0,0.4376425881)+-(0,0.0188792272)(64.0,0.4937238353)+-(0,0.0107355933)};
                \addplot plot [error bars/.cd, y dir=both, y explicit] coordinates{(8.0,0.6240234375)+-(0,0.1090212961)(16.0,0.8315939832)+-(0,0.0622250634)(32.0,0.9538671875)+-(0,0.0288359967)(64.0,1.0247865499)+-(0,0.0165278258)};
            \end{axis}
        \end{tikzpicture}
    \end{center}
    \caption{The mean diversity (total imbalance) and its standard deviation at the moment when the \gsemo or the \gsemod covers the whole Pareto front for the first time. Solid lines correspond to the \gsemod optimizing the total imbalance measure. Dashed lines correspond to the \gsemod optimizing the sorted imbalances vector. Dashed-dotted lines correspond to the \gsemo which does not optimize any diversity measure. For highlighting the differences, each value is normalized by the minimum possible total imbalance for the corresponding values of $n$ and $k$ computed as in Lemma~\ref{lem:imbalance} and then decreased by one.}
    \label{fig:total-imbalance-experiment}
\end{figure}

In Figure~\ref{fig:total-imbalance-experiment} we see that the \gsemod, independently of which diversity measure it optimizes, has a diversity very close to the optimal value when it covers the Pareto front. In contrast, the \gsemo has a diversity which is far from optimal (except for $k = 2$, when we have a unique Pareto-optimal population). This indicates that diversity optimization works very well also in the first phase of optimization, even before we find the whole Pareto front. We believe that this is a strong argument for implementing diversity-improving tie-breaking mechanisms in classic optimization: they help obtain a much more diverse population at almost no cost. In addition, in crossover-based algorithms it might help to exploit the true potential of crossover, as it was done in~\cite{DangFKKLOSS16}.   

\subsection{Starting with the Worst Diversity}

To better understand diversity optimization and to isolate it from Pareto optimization, we performed experiments where the initial population covers the Pareto front, but has the worst possible diversity. We used the same values of $n$ and $k$ as in Section~\ref{sec:experiments-total-runtime}. To make a population with the worst diversity, we maximized the imbalance of each position. For this we used Lemma~\ref{lem:imbalance}: while we cannot change the bits values that determine fitness of each individual in the population, for each position $i$ we can set the $i$-th bit of all individuals with fitness in $M(i)$ to either one or zero, and we can choose the value which maximizes the imbalance. The results of this experiment are shown in Figure~\ref{fig:worst-div-experiment}.

\begin{figure}[t]
    \begin{center}
        \begin{tikzpicture}
            \begin{axis}[width=0.95\linewidth, height=0.28\textheight,
                at={(0, 0)},anchor=outer south west,
                cycle list name=tikzcycle, grid=major,  xmode=log, log base x=2, legend style={at={(0.5, 1.2)}, anchor=south,legend columns = 5, inner sep=5pt},
                legend cell align={left},
                xlabel={Problem size $n$}, ylabel={Runtime $/kn^2\ln(n)$}]

                \addplot plot [error bars/.cd, y dir=both, y explicit] coordinates{(8.0,0.0000)+-(0,0.0000)(16.0,0.0000)+-(0,0.0000)(32.0,0.0000)+-(0,0.0000)(64.0,0.0000)+-(0,0.0000)(128.0,0.0000)+-(0,0.0000)};
                \addlegendentry{$k = 2$};
                \addplot plot [error bars/.cd, y dir=both, y explicit] coordinates{(8.0,0.8163)+-(0,0.3557)(16.0,1.0062)+-(0,0.3169)(32.0,1.1054)+-(0,0.2446)(64.0,1.1188)+-(0,0.2541)(128.0,1.0767)+-(0,0.1801)};
                \addlegendentry{$k = 4$};
                \addplot plot [error bars/.cd, y dir=both, y explicit] coordinates{(8.0,0.0000)+-(0,0.0000)(16.0,1.0062)+-(0,0.3169)(32.0,1.1224)+-(0,0.2750)(64.0,1.1889)+-(0,0.2045)(128.0,1.3373)+-(0,0.3609)};
                \addlegendentry{$k = \sqrt{n}$};
                \addplot plot [error bars/.cd, y dir=both, y explicit] coordinates{(8.0,0.8163)+-(0,0.3557)(16.0,0.9814)+-(0,0.2652)(32.0,1.0660)+-(0,0.2594)(64.0,1.6893)+-(0,0.3653)};
                \addlegendentry{$k = n/2$};
                \addplot plot [error bars/.cd, y dir=both, y explicit] coordinates{(8.0,0.4531)+-(0,0.1719)(16.0,0.6517)+-(0,0.1532)(32.0,0.9425)+-(0,0.1785)(64.0,1.6397)+-(0,0.2822)};
                \addlegendentry{$k = n$};
                \addplot plot [error bars/.cd, y dir=both, y explicit] coordinates{(8.0,0.0000)+-(0,0.0000)(16.0,0.0000)+-(0,0.0000)(32.0,0.0000)+-(0,0.0000)(64.0,0.0000)+-(0,0.0000)(128.0,0.0000)+-(0,0.0000)};
                \addplot plot [error bars/.cd, y dir=both, y explicit] coordinates{(8.0,0.7893)+-(0,0.3332)(16.0,0.9969)+-(0,0.3275)(32.0,1.0750)+-(0,0.3021)(64.0,1.1145)+-(0,0.2657)(128.0,1.1118)+-(0,0.2493)};
                \addplot plot [error bars/.cd, y dir=both, y explicit] coordinates{(8.0,0.0000)+-(0,0.0000)(16.0,0.9969)+-(0,0.3275)(32.0,1.1065)+-(0,0.2401)(64.0,1.2211)+-(0,0.2414)(128.0,1.2163)+-(0,0.2004)};
                \addplot plot [error bars/.cd, y dir=both, y explicit] coordinates{(8.0,0.7893)+-(0,0.3332)(16.0,0.9488)+-(0,0.2385)(32.0,0.9725)+-(0,0.1872)(64.0,1.0982)+-(0,0.3074)};
                \addplot plot [error bars/.cd, y dir=both, y explicit] coordinates{(8.0,0.3977)+-(0,0.1277)(16.0,0.5515)+-(0,0.1296)(32.0,0.6149)+-(0,0.1049)(64.0,0.6774)+-(0,0.0870)};
            \end{axis}
        \end{tikzpicture}
    \end{center}
    \caption{The normalized runtimes of the \gsemod when optimizing the total imbalance diversity (solid lines) and the sorted imbalances vector (dashed lines). The initial population is a population with the worst possible diversity, and it is the same in all runs. All runtimes are normalized by $kn^2\ln(n)$, which is asymptotically the same as the upper bound shown in Theorem~\ref{thm:total-balance}.}
    \label{fig:worst-div-experiment}
\end{figure}

In this plot we see that despite the differences in the diversity measures, the performance of the algorithm is very similar when optimizing any of these two measures. Interestingly, for $k = n$ we even observe a small advantage of optimizing the measure of sorted imbalances vector, which is surprising, since this measure seems more complicated to optimize. We also see from this plot that our upper bound in Theorem~\ref{thm:total-balance} is likely to be tight, and that the true bound from Theorem~\ref{thm:sorted-balances} is likely to be similar to the bound for total imbalance.

\section{Conclusion}

In this paper, we have shown that optimizing diversity with EAs in a multi-objective setting might be easy compared to the time needed for the computation of the Pareto front. We showed that a simple tie-breaking rule implemented into the \gsemo can effectively find the best possible diversity of a Pareto-optimal population. This lines up with the result of \cite{DBLP:conf/foga/Antipov0N23} for the \oneminmax problem, even though the main source of diversity improvements is different in their setting (in~\cite{DBLP:conf/foga/Antipov0N23} and also in~\cite{DBLP:conf/gecco/DoerrGN16} the proofs relied on two-bits flips which improve the diversity).
Our analysis is also the first one performed on a multi-objective problem with more than two objectives, which demonstrates that evolutionary algorithms can be effective within such a multi-dimensional domain, where the main factor which slows down the optimization is usually the large size of the Pareto front. Both theoretical and empirical results strongly support the idea of implementing tie-breaking rules for diversity enhancement into multi-objective algorithms, since it helps to obtain a much more diverse population without slowing down the optimization itself. 

The results, however, raise a question regarding which diversity measures are the fastest to optimize. As we see from our results, although the two considered measures share the set of populations which have the optimal diversity, they are optimized with the \gsemod in different ways. In practice the difference might be even larger, since the diversity is also optimized in the earlier stages of optimization, that is, before we find a Pareto-optimal population, and this difference might be of particular interest for using EDO in practice.

\section*{Acknowledgements} This work has been supported by the Australian Research Council (grants DP190103894 and FT200100536), as well as by the European Union (ERC CoG ``dynaBBO'', grant no.~101125586) and Alliance Sorbonne Universit{\'e} (project number EMERGENCE 2023 RL4DAC).

\end{document}